\documentclass{article}
\usepackage{microtype}
\usepackage{graphicx}
\usepackage{booktabs} 
\usepackage{hyperref}
\usepackage{amsmath}
\usepackage{amssymb}
\usepackage{mathtools}
\usepackage{amsthm}

\usepackage[capitalize,noabbrev]{cleveref}
\usepackage{algorithmic}
\usepackage{algorithm}

\usepackage[accepted]{icml2024}

\theoremstyle{plain}
\newtheorem{theorem}{Theorem}[section]
\newtheorem*{theorem*}{Theorem}

\theoremstyle{definition}

\theoremstyle{remark}

\usepackage[textsize=tiny]{todonotes}
\usepackage{amsmath}
\usepackage{amsthm}
\usepackage{thmtools}
\usepackage{thm-restate}
\usepackage{graphicx}
\usepackage{adjustbox}
\usepackage{subcaption}
\usepackage{breqn}

\usepackage{paralist}
\usepackage{listings}
\usepackage{enumitem}

\icmltitlerunning{Understanding Inter-Concept Relationships}

\begin{document}

\twocolumn[
\icmltitle{Understanding Inter-Concept Relationships in Concept-Based Models}

\icmlsetsymbol{equal}{*}

\begin{icmlauthorlist}
\icmlauthor{Naveen Raman}{cmu}
\icmlauthor{Mateo Espinosa Zarlenga}{cam}
\icmlauthor{Mateja Jamnik}{cam}
\end{icmlauthorlist}

\icmlaffiliation{cmu}{Carnegie Mellon University}
\icmlaffiliation{cam}{University of Cambridge}

\icmlcorrespondingauthor{Naveen Raman}{naveenr@cmu.edu}

\icmlkeywords{Machine Learning, ICML}

\vskip 0.3in
]

\newcommand{\nrcomment}[1]{}
\newcommand{\mecomment}[1]{{\color{blue} Mateo: {#1}}}
\newcommand{\kccomment}[1]{{\color{purple} Katie: {#1}}}

\printAffiliationsAndNotice{}

\begin{abstract}
Concept-based explainability methods provide insight into deep learning systems by constructing explanations using human-understandable concepts. 
While the literature on human reasoning demonstrates that we exploit relationships between concepts when solving tasks, it is unclear whether concept-based methods incorporate the rich structure of inter-concept relationships. 
We analyse the concept representations learnt by concept-based models to understand whether these models correctly capture inter-concept relationships.
First, we empirically demonstrate that state-of-the-art concept-based models produce representations that lack stability and robustness, and such methods fail to capture inter-concept relationships. 
Then, we develop a novel algorithm which leverages inter-concept relationships to improve concept intervention accuracy, demonstrating how correctly capturing inter-concept relationships can improve downstream tasks. 

\end{abstract}

\section{Introduction}
\label{sec:introduction}
Explainability methods construct explanations for predictions made by deep learning systems. 
One approach for generating such explanations is via high-level units of information referred to as ``\textit{concepts}''~\citep{koh2020concept}. 
For example, a model's classification of a fruit as an ``apple'' can be explained because the model detected the concepts of ``\textit{red colour}'' and ``\textit{round shape}''.
These models have been applied to tasks such as human-AI teaming~\citep{zarlenga2023learning}, uncertainty quantification~\citep{kim2023probabilistic}, and model debugging~\citep{bontempelli2022concept}. 

Many existing concept-based models~\citep{koh2020concept, kim2018interpretability, zarlenga2022concept} predict concepts independently, despite the prevalence of interrelated concepts in real-world situations. 
For example, birds with ``\textit{grey wings}'' tend to have ``\textit{grey tails}'', and patients who have ``\textit{lung lesions}'' tend to be on ``\textit{support devices}''.
Learning from inter-concept relationships better mimics the way humans process information~\cite{mcclelland2003parallel} and could assist with downstream tasks. 
However, leveraging these relationships can be difficult because (1)~concept labels tend to be noisy, as annotations can be imperfect~\citep{collins2023human}, and (2)~explainability methods are inherently unstable~\citep{brown2021brittle,dombrowski2019explanations}. 

\begin{figure*}[t!]
    \centering
    \includegraphics[width=0.9\textwidth]{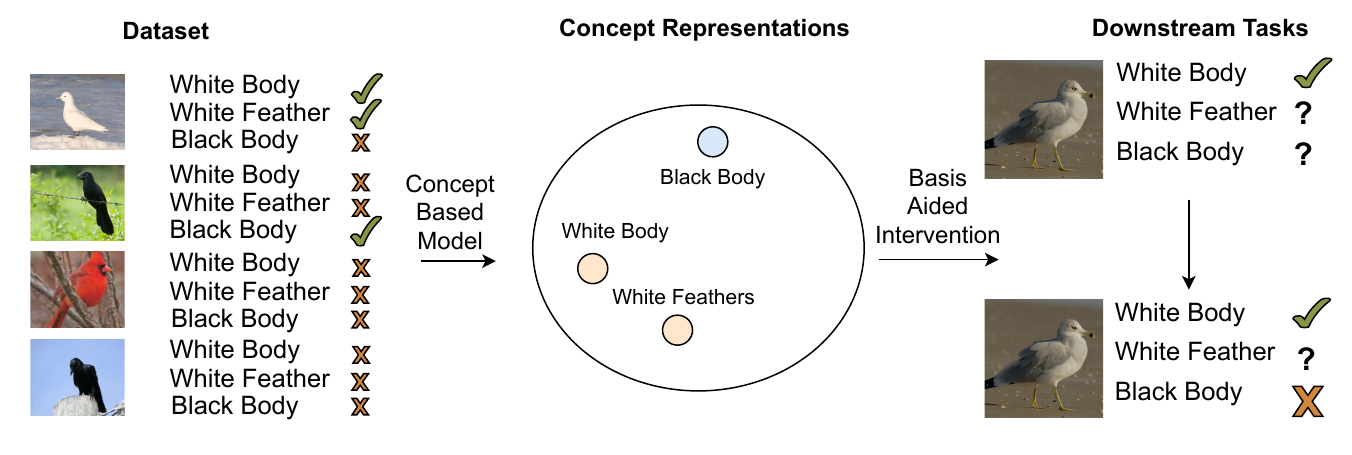}
    \caption{We analyse whether concept-based models capture inter-concept similarities by studying their learnt ``concept vector'' representations.
    We demonstrate that learning representations which properly capture inter-concept relationships can help with downstream tasks.
    For example, these relationships can assist with test-time concept interventions by imputing uncertain concepts via known concept labels (e.g., we can determine the concept ``\textit{Black body}'' from ``\textit{White body}'').}
    \label{fig:main}
    \vspace*{-2mm}
\end{figure*}

We study inter-concept relationships in concept-based models to understand how concept-based models capture inter-concept relationships, an often overlooked area in prior work. 
By analysing learnt representations, we surprisingly find that state-of-the-art concept-based models may fail to capture known inter-concept relationships.
We then construct a novel algorithm which exploits inter-concept relationships to improve the effectiveness of human-AI \textit{concept interventions} -- where experts correct some mispredicted concepts -- highlighting how leveraging inter-concept relationships can improve downstream tasks.
We illustrate our approach in Figure~\ref{fig:main} and summarise our contributions: 

\begin{enumerate}[topsep=0pt,leftmargin=11pt,itemsep=0pt]
    \item We analyse the concept representations constructed by existing concept-based models and show that, unexpectedly, these representations may fail to capture known inter-concept relationships\footnote{Our code is available here: \url{https://github.com/naveenr414/Concept-Learning}.}.

    \item We propose an algorithm that exploits inter-concept relationships to improve test-time concept interventions. 
    \item We theoretically show that leveraging inter-concept relationships can improve concept intervention performance and validate this result empirically.  
\end{enumerate}\looseness-1

\section{Related Works}
\label{sec:related}
\paragraph{Concept-based Explainability}  
Developing explainability methods is challenging due to potentially conflicting goals \citep{rudin2022interpretable,lipton2018mythos} including eliciting trust~\citep{shen2022trust}, accurately representing model reasoning~\citep{lipton2018mythos}, and efficiently generating explanations~\citep{langer2021we}. 
Concept-based explainability methods aim to cover these desiderata by developing explanations for model predictions using high-level units of information called \textit{concepts}~\citep{kim2018interpretability, ghorbani2019towards, chen2020concept, cme}.
Recent methods in this field, such as Concept Bottleneck Models (CBMs)~\citep{koh2020concept}, Concept Embedding Models (CEMs)~\citep{zarlenga2022concept}, and recently proposed variants~\cite{havasi2022addressing, post_hoc_cbms, label_free_cbms}, put forth architectures that construct concept-based explanations by first predicting the presence of concepts and then predicting a label based on these concept predictions. 
In this work, we focus on understanding the inter-concept relationships learnt by such models, as such relations are key components of human-like reasoning.

\paragraph{Inter-Concept Relationships} 
The use of inter-concept relationships for human reasoning has been studied in cognitive science as a model of cognition and understanding~\citep{chater2010bayesian,griffiths2007unifying,mao2019neuro}. 
Generally, graph-based structures are a common way of relating large amounts of information in natural language processing~\citep{alsuhaibani2019joint,mikolov2013efficient}, database management~\citep{jonyer2001graph}, and knowledge graphs~\citep{hogan2021knowledge}.  
In this work, we explore whether concept representations learnt by state-of-the-art concept-based models properly capture inter-concept relationships, an important property that, to the best of our knowledge, has not been previously studied.

\paragraph{Analysing Concept-Based Models} 
Our work fits into the wider literature that analyses the behaviour and failure modes of concept-based models. 
Prior work in this space has analysed concept-task leakage, where task information is leaked into concept predictors, thereby leading to erroneous concept predictions~\cite{mahinpei2021promises,marconato2022glancenets,havasi2022addressing}, and such an issue could potentially jeopardise the learnt inter-concept relationships.
Other work has investigated the robustness of concept predictors and shown that their predictions fail to truly reflect the presence of concepts due to concept correlations~\cite{raman2024concept}. 
Both lines of work demonstrate the fragility of concept-based models. 
We build on these works by investigating concept-based models through the lens of inter-concept relationships. 


\section{Defining Inter-Concept Relationships}
\label{sec:hierarchies}

\paragraph{Introducing Concepts} 
\looseness-1
Concept-based learning is a supervised learning setup where we are given a set of training samples $X = \{\mathbf{x}^{(i)} \in \mathbb{R}^{m}\}_{i=1}^n$ and corresponding labels $Y = \{y^{(i)} \in \{1, \cdots, L\} \}_{i=1}^n$ annotated with vectors of high-level concepts $C = \{\mathbf{c}^{(i)} \in \{0,1\}^{k}\}_{i=1}^n$.
In this setup, the $i$-th data point $\mathbf{x}^{(i)}$ has an associated set of $k$ binary concepts (either inactive or active) where the activation of the $j$-th concept is denoted by $\mathbf{c}^{(i)}_j$.
For example, when learning to predict a bird's species $y^{(i)}$ from its image $\mathbf{x}^{(i)}$, $\mathbf{c}^{(i)}_{1}$ could represent the concept ``\textit{white tail colour}''. 

\looseness-1
Concept labels can be used to develop deep neural network architectures which make both task and concept predictions. 
One such architecture is a Concept Bottleneck Model (CBM)~\citep{koh2020concept}, which uses a \textit{concept predictor},~$g$, to predict concepts $\hat{\mathbf{c}}$ from an input $\mathbf{x}$, and a \textit{label predictor},~$f$, which predicts labels~$\hat{y}$ from concepts~$\hat{\mathbf{c}}$.
Their two-stage architecture allows experts to \textit{intervene on concept predictions} at test time by correcting a subset of mispredicted concepts. 
This procedure enables human-AI teams to improve a CBM's test accuracy~\citep{koh2020concept,chauhan2022interactive}. 
More recent extensions of CBMs, such as concept embedding models (CEMs) \citep{zarlenga2022concept}, generalise a CBM's bottleneck by using high-dimensional embeddings. 

\paragraph{Introducing Concept Bases} 
In this paper, we study inter-concept relationships by analysing the concept representations learnt by concept-based models, a previously unstudied aspect.
Understanding similarities between concept representations can uncover whether concept-based models pick up on inter-concept relationships, and can also help with downstream applications (see Figure~\ref{fig:main}).  
For example, the representation for the concept ``\textit{yellow head}'' should be closer to that of ``\textit{yellow neck}'' than for ``\textit{green tail}''.
Such an analysis crucially serves as a sanity check that a model properly captures known patterns.

Concept-based models implicitly learn a set of concept vectors, which we call a \textit{concept basis} $B = \{\mathbf{v}^{(j)} \in \mathbb{R}^d\}_{j =1}^k$, where each vector $\mathbf{v}^{(j)}$ is a $d$-dimensional representation of concept~$j$. 
We note that vectors in the concept basis are not necessarily independent, and instead view the concept basis as a collection of vectors which defines some set of concepts. 
These bases can take many forms, including Concept Activation Vectors (CAVs)~\citep{kim2018interpretability} and concept embeddings learnt by CEMs \citep{zarlenga2022concept}. 
Inter-concept relationships can provide a structure which practitioners can use to better understand the reasoning behind a model's predictions~\citep{bansal2021does}.
Additionally, as we will show later, understanding inter-concept relationships such as mutual exclusivity can be exploited for error correction during inference. 


\section{Designing Metrics to Analyse Inter-Concept Relationships}
\label{sec:desiderata}
We propose a set of desiderata for well-calibrated concept representations to help evaluate whether concept-based models capture inter-concept relationships. 
These desiderata enable us to contrast concept bases learnt by different concept-based models and explore whether these representations properly capture inter-concept relationships. 
Taking inspiration from previous desiderata for explainable AI methods~\citep{hedstrom2022quantus}, we argue that well-calibrated concept representations should be: (1)~\textbf{stable} -- they should capture similar inter-concept relationships across random seeds, (2)~\textbf{robust} -- the inter-concept relationships captured should not vary based on small input perturbations, (3)~\textbf{responsive} -- the inter-concept relationships should vary when the input is significantly altered, and (4)~\textbf{faithful} -- the inter-concept relationships should accurately reflect any known inter-concept relationships in a dataset. 



\vspace*{-2mm}
\paragraph{Distances Between Concept Bases} 
Measuring the above desiderata requires a way to measure the similarity of different concept bases, so we can capture variations across factors such as the dataset and random seed.
We define a distance metric between concept bases based on the similarity of inter-concept relationships captured by each basis. 
Let $\delta_{b}(B,B^\prime)$ be the distance between concept bases $B$ and $B^\prime$. 
To compute $\delta_{b}$, we calculate the overlap between the $t$ most similar concepts to $j$ in $B$ and those in $B^\prime$. 

Formally, let $\delta_{v}$ be a distance metric between concept vectors, like the $\ell_2$-norm. 
Then, for a concept vector $\mathbf{v}^{(j)} \in B$, we denote the $t$-nearest concept vectors as $N_B(\mathbf{v}^{(j)})$, where $t$ is a hyperparameter. 
We then compute overlap between $N_B(\mathbf{v}^{(j)})$ and $N_{B^\prime}(\mathbf{v}^{(j^\prime)})$, averaged across concepts:
\begin{align*}
    \scriptstyle
    \delta_b(\{\mathbf{v}^{(1)} \cdots \mathbf{v}^{(k)}\}, \{\mathbf{v}^{\prime (1)} \cdots \mathbf{v}^{\prime (k)}\}) := 1 - \frac{1}{k} \sum_{i=1}^{k} \frac{|N(\mathbf{v}^{(i)}) \cap N(\mathbf{v}^{\prime (i)})|}{t}.
\end{align*}

\paragraph{Metrics for Concept Vectors}
We quantify our desiderata as follows (details in Appendix~\ref{sec:metric_details}):
\begin{enumerate}[topsep=0pt,leftmargin=11pt,itemsep=0pt]
    \item \textbf{Stability} can be measured as $1-\mathbb{E}[\delta_b(B, B^{\prime})]$ where $B$ and $B^{\prime}$ are sampled independently from the same concept-based model with different training seeds (higher values are more stable). In practice, we estimate this through Monte Carlo sampling. 
    \item \textbf{Robustness} is measured as $1-\delta_b(B, B^{\prime})$, where $B$ is a basis learnt from an unperturbed baseline dataset, while $B^{\prime}$ is a basis learnt from a slightly perturbed dataset.  
    \item \textbf{Responsiveness} can be computed by constructing $B$ from a baseline dataset and measuring $\delta_b(B, B^{\prime})$, where $B^{\prime}$ uses concept bases extracted from a corrupted dataset. $B^{\prime}$ in robustness involves small amounts of noise, while $B^{\prime}$ in responsiveness involves large amounts of noise. 
    \item \textbf{Faithfulness} measures whether similarities between concepts in a dataset mirror similarities between concepts in a concept basis. 
    In other words, concepts which have a similar impact on the task label should have similar representations. 
    For example, the presence of either ``\textit{white body}'' or ``\textit{white head}'' increases the probability that a bird is predicted to be a ``pigeon.'' 
    
    Formally, this can be computed by constructing a set of vectors $\{\mathbf{s}^{(1)}, \cdots, \mathbf{s}^{(k)}\}$ where the $l$-th entry of $\mathbf{s}^{(j)} := [\mathbf{s}^{(j)}_1, \cdots, \mathbf{s}^{(j)}_L]^T$ indicates the importance of concept $j$ on task label $l$. 
    We then evaluate the faithfulness of a concept basis $B$ by considering the set $B_s = \{\mathbf{s}^{(1)}, \cdots, \mathbf{s}^{(k)}\}$ as a concept basis and computing faithfulness as $1 - \delta_b(B_s, B)$. 
    To compute $\mathbf{s}^{(j)}_l$, we take inspiration from Shapley values~\citep{shapley1953value}, and compare the difference in predictions for label $l$  on data points with and without concept $j$.
    Given a concept predictor $g$ and a label predictor $f$, we compute: 
    \begin{align*} 
        s^{(j)}_l := \sum_{i \in \mathcal{A}_j} f(g(\mathbf{x}^{(i)}))_j - \sum_{i \in \{1, \cdots, n\} \backslash \mathcal{A}_j} f(g(\mathbf{x}^{(i)}))_j.
    \end{align*} 
\end{enumerate}

We evaluate and justify these metrics as a method to evaluate concept-based methods through two studies. 
In Appendix~\ref{sec:synthetic}, we construct a synthetic scenario where we demonstrate that a well-designed concept-based model achieves higher scores for stability and robustness when compared to poorly designed concept-based models. 
In Appendix~\ref{sec:ois}, we justify the creation of these metrics through the lens of concept leakage~\cite{mahinpei2021promises,zarlenga2023towards}, where we demonstrate that better scores on each of these metrics correlate with concept-based models which exhibit lower concept leakage.  
\section{Do Concept-based Models Capture Known Inter-Concept Relationships?}
\label{sec:experiments}
\subsection{Discovering Concept Bases} 
We analyse the representations underlying various methods for concept-based learning by focusing on three key methods. 
These methods capture both popular concept-based models (\textit{TCAV}, \textit{CEM}), and algorithms for learning representations in structured datasets (\textit{Concept2Vec}): 

\begin{enumerate}[topsep=0pt,leftmargin=11pt,itemsep=0pt]
    \item \textbf{TCAV Vectors:} We compute concept bases through the Testing with Concept Activation Vectors (TCAV) algorithm~\cite{kim2018interpretability}. For each concept, this approach computes intermediate activations from a trained model and learns a linear separator in model space for points with and without a concept. This separator, known as a concept activation vector, serves as a high-dimensional representation of the concept (see Appendix~\ref{sec:model_details}). 
    \item \textbf{CEM Embeddings:} For each input sample $\mathbf{x}^{(i)}$, a CEM learns a high-dimensional representation of each concept that enables simultaneous prediction of the concept $\mathbf{c}^{(i)}$ and task label $y^{(i)}$~\cite{zarlenga2022concept}. 
    We construct ``global'' representations of each concept by letting the $j$-th concept vector be $\mathbf{v}_j = \sum_{i \in \mathcal{A}_j} \hat{\mathbf{z}}^{(i)}_j/|\mathcal{A}_j|$ where $\mathcal{A}_j = \{i^\prime | \mathbf{c}^{(i^\prime)}_j = 1 \}$ is the set of training samples with the $j$-th concept being active and $\hat{\mathbf{z}}^{(i)}_j \in \mathbb{R}^d$ is a CEM's predicted concept embedding for concept $j$ and sample $\mathbf{x}^{(i)}$. 
    In essence, we compute the mean concept embedding across all samples in the training set which contain the concept (details in Appendix~\ref{sec:model_details}).
    \item \textbf{Concept2Vec:} We employ an algorithm similar to word2vec~\citep{mikolov2013efficient} to learn concept representations based on inter-concept co-occurrences. We retrieve representations using a Skipgram architecture (details in Appendix~\ref{sec:model_details}).
\end{enumerate}

\paragraph{Ground-Truth Baseline: Label Bases}
To contextualise the performance of concept bases, we introduce the \textit{label basis} as a ground-truth baseline, allowing us to understand the gap across metrics between existing concept-based models and an idealised concept basis. 
This baseline achieves each of the desiderata mentioned in Section~\ref{sec:desiderata} and captures all known inter-concept relationships. Therefore, it serves as a good upper bound on the performance for each metric. 

We define the label basis based on concept co-occurrences under the assumption that similar concepts frequently co-occur. 
Formally, we define each label vector as $\mathbf{v}^{(i)} :=[\textbf{c}^{(1)}_{j},\textbf{c}^{(2)}_{j} \cdots \textbf{c}^{(n)}_{j}]^T \in \{0, 1\}^n$, where $\mathbf{v}^{(j)}$ is an $n$-dimensional vector whose $i$-th entry represents whether the $j$-th concept is active for training sample $\textbf{x}^{(i)}$. 
If $\mathbf{v}^{(j)}$ and $\mathbf{v}^{(j')}$ have small distance, then concepts $j$ and $j'$ co-occur frequently. 
Because concept co-occurrences are averaged across all data points, perturbations to a small subset of inputs should minimally impact concept co-occurrences, thereby not changing the similarities between label bases. 
Similarly, large corruptions to datasets should significantly alter concept co-occurrences, thereby changing the resulting label basis. 
The label basis allows us to understand the performance of other concept bases by placing an upper bound on their performance across metrics. 

\subsection{Datasets and Experimental Setup} 
We evaluate concept-based models using the metrics described in Section~\ref{sec:desiderata} on the following synthetic (MNIST, dSprites) and non-synthetic (CUB, CheXpert) vision tasks (more details in Appendix~\ref{sec:dataset_details}): 
\begin{enumerate}[topsep=0pt,leftmargin=11pt,itemsep=0pt]
    \item 
    \textbf{Coloured MNIST}~\citep{arjovsky2019invariant} is a vision dataset where each sample is a coloured hand-written digit. There are ten digit and ten colour concepts, with each digit being paired with a colour, leading to ten digit-colour combinations. This task allows us to study concept representations in a controlled setting. 
    \item 
    \textbf{dSprites}~\citep{dsprites17} is a vision dataset where each object has a shape, location, size, and orientation. We use these attributes as ground-truth concept annotations and construct a task label for this dataset as the base-10 representation of the sample's binary concept vector. We select ten combinations of concepts and sample images from these to study concept bases.     
    \item 
    \textbf{CUB}~\citep{wah2011caltech} is a bird image dataset in which each sample is annotated with its species. We are additionally provided with concept attributes describing different properties such as the bird's size, wing colour, head colour, etc. Of the 312 provided binary attributes, we select the same 112 attributes as~\citet{koh2020concept} to improve class balancing across concepts.  
    \item 
    \textbf{CheXpert}~\citep{irvin2019chexpert} is a medical chest radiograph dataset with concepts annotated from medical notes. As done by~\citet{chauhan2022interactive}, we use 13 of its annotations as concepts and predict ``no condition''. 
\end{enumerate}

\subsection{Capturing Simple Relationships (MNIST)}
Since the colour and digit concepts are perfectly correlated in the coloured MNIST dataset, this enables a simple setup to evaluate whether concept bases can recover these relationships. 
To quantify this, we evaluate the fraction of concepts where the most similar concept matches the ground truth, measuring similarity through the concept distance metric, $\delta_{v}$. 
We expect that concept bases should recover digit-colour similarities (e.g., ``\textit{digit 2}'' is most similar to ``\textit{colour 2}''). 

In Table~\ref{tab:mnist} we observe that CEM bases surprisingly fail to recover digit-colour pairs for 13\% of concepts.  
This implies that even in simple scenarios, the concept vectors arising from CEM bases fail to capture straightforward inter-concept relationships.
We additionally find that Concept2Vec, TCAV, and label bases successfully recover the similarity between digit and colour concepts for all pairs, resulting in high concept agreement.

\begin{table}[h]
    \centering 
    \caption{CEM concept bases fail to correctly capture inter-concept relationships, as they fail to enforce similar representations for the colour and digit concepts in coloured MNIST. 
    This leads to imperfect ($<100\%$) concept agreement between colour and digit concepts.}
    \label{tab:mnist}
    \begin{tabular}{@{}lr@{}}
    \toprule
    Basis & \multicolumn{1}{l}{Concept Agreement ($\uparrow$)} \\ \midrule
    TCAV & 100\% $\pm$ 0\% \\ 
    CEM & 87\% $\pm$ 19\% \\ 
    Concept2Vec & 100\% $\pm$ 0\% \\ 
    Label & 100\% $\pm$ 0\% \\ \bottomrule
    \end{tabular}
\end{table}

\begin{figure*}
    \centering

    \includegraphics[width=0.93\textwidth]{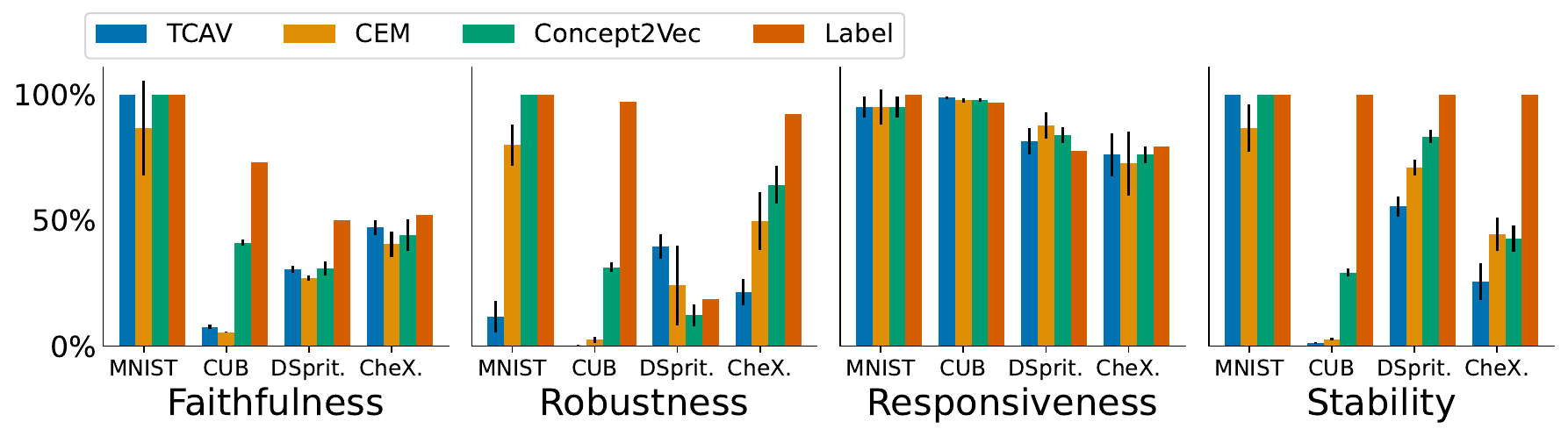}

    \caption{
        Representations from TCAV and CEM achieve significantly lower scores on faithfulness, robustness, and stability when compared with the label basis, highlighting the instability of these representations, an unexpected shortcoming. 
    }
    \label{fig:faithfulness}
\end{figure*}

\subsection{Metric-based Evaluation}
Figure~\ref{fig:faithfulness} reports the performance of various concept bases for our four metrics from Section~\ref{sec:desiderata}. We discuss them next.



\paragraph{Concept bases extracted from Label concept vectors perform well} 
In Figure~\ref{fig:faithfulness}, we find that label concept bases perform well across all metrics and datasets (except robustness on dSprites, which we discuss below). 
This result is best seen in CUB, where the label basis exhibits higher faithfulness than all other bases.
This trend validates using the label basis as a ground truth.
The gap between Label and other bases highlights the inability of existing concept-based models to pick up on inter-concept relationships. 

\paragraph{CEM and TCAV bases exhibit significant instability} 
On CUB and dSprites, TCAV and CEM bases exhibit low stability and robustness. 
This may be due to inherent fluctuations in each method:
TCAV models build linear separators via model activations, which can fluctuate across even similar data points, while the underlying representations in CEM models might fluctuate across iterations and data points. 
This implies that TCAV and CEM cannot consistently recover the same inter-concept relationships across trials. 

\paragraph{dSprites dataset is challenging across all Concept Bases} 
All concept bases struggle to capture inter-concept relationships in dSprites, highlighting the difficulty of developing good representations. 
The concepts in the dSprites dataset are weakly correlated, in contrast to the strong correlation found in CUB and MNIST, making it difficult to find significant inter-concept relationships. 
For robustness and faithfulness on dSprites, all concept bases achieve scores under $0.5$, which is lower than the scores for any other dataset. 

\subsection{Qualitative Evaluation}
\label{sec:qualitative}
We visualise concept bases by hierarchically clustering concept vectors to understand the inter-concept relationships recovered. 
We employ Ward's hierarchical clustering~\citep{ward1963hierarchical}, though algorithms such as single linkage clustering~\cite{gower1969minimum}, produce similar visualisations.

\begin{figure*}
    \centering
    \begin{subfigure}{90pt}
        \caption*{\textbf{TCAV}}
    \includegraphics[width=90pt]{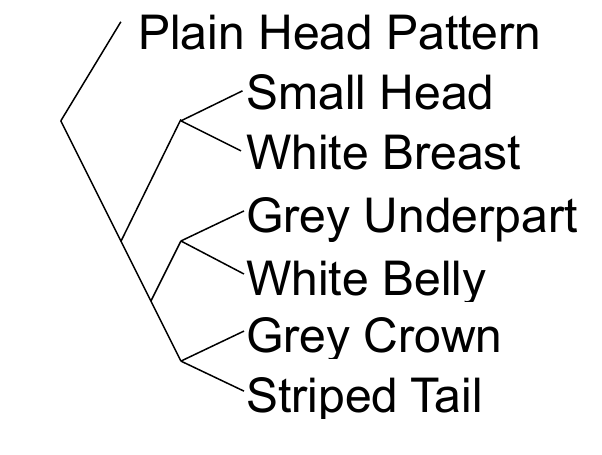}
    \end{subfigure}
    \hfil
    \begin{subfigure}{90pt}
        \caption*{\textbf{CEM}}
    \includegraphics[width=90pt]{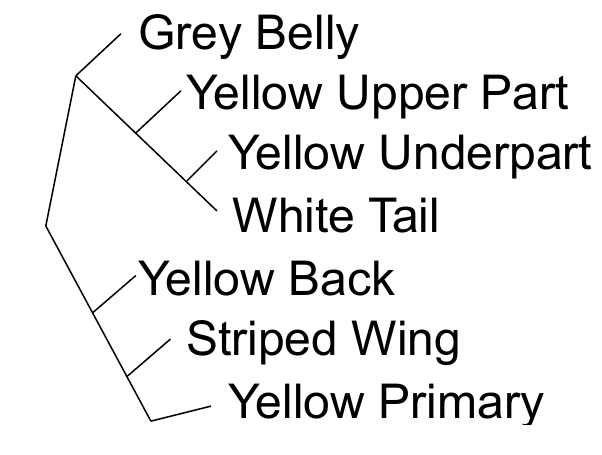}
    \end{subfigure}
    \hfil
    \begin{subfigure}{90pt}
        \caption*{\textbf{Concept2Vec}}
    \includegraphics[width=90pt]{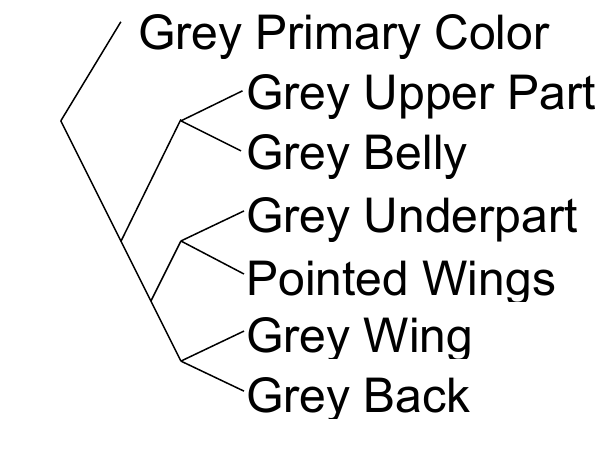}
    \end{subfigure}
    \hfil
    \begin{subfigure}{90pt}
        \caption*{\textbf{Label}}
    \includegraphics[width=90pt]{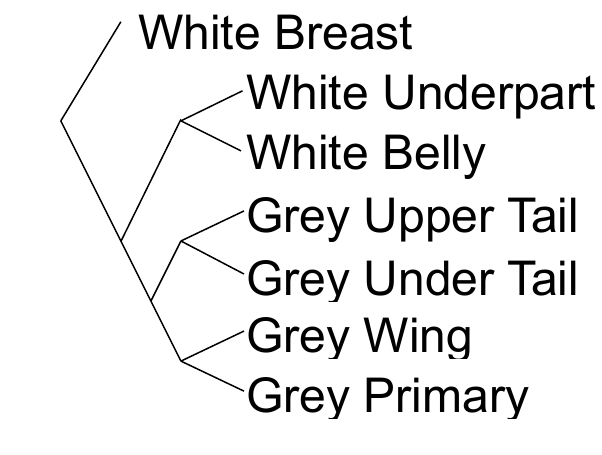}
    \end{subfigure}

    \caption{Different concept-based models learn different concept bases on the CUB dataset. Some bases reflect the semantic similarity between concepts (label, Concept2Vec), while others lack any pattern or structure (CEM, TCAV).}
    \label{fig:hierarchy}
\end{figure*}

Our qualitative evaluation confirms the findings from our quantitative metrics. 
Label bases accurately capture inter-concept relationships across datasets, with semantically similar concepts being adjacent in the hierarchy (Figure~\ref{fig:dataset}). 
Conversely, the CEM and TCAV concept bases fail to capture known inter-concept relationships, which is seen through the lack of structure in their visualised representations. 
Our results cast some on the reliability of these representations as faithfully representing the concepts. 
\section{Leveraging Inter-Concept Relationships for Concept Intervention}
We build on our analysis from Section~\ref{sec:experiments} to show that concept bases which recover inter-concept relationships can be useful for downstream tasks. 
We theoretically demonstrate the effectiveness of label bases for concept intervention, then empirically validate this through a novel algorithm which leverages concept bases to improve concept intervention. 

\subsection{Concept Interventions}
\label{sec:hierarchy_intervention}

CBMs, CEMs, and their more recent variants~\cite{havasi2022addressing, kim2023probabilistic, zarlenga2023learning} allow experts to correct concept mispredictions at test-time to improve task performance. 
This process, called \textit{concept intervention}, occurs, for example, when a clinician observes that a concept prediction disagrees with their analysis and corrects that misprediction. 
We can formalise this procedure by considering the problem of classifying $(\mathbf{x}^{(i)}, y^{(i)}, \mathbf{c}^{(i)})$ by first trying to predict all $k$ concepts $\mathbf{c}^{(i)}$, then having an expert impute ground-truth concept values for $r \leq k$ of these concepts.
The label predictor is then re-run using this mixture of predicted and ground-truth concept values.   



\textbf{Label Bases and Concept Intervention} \hspace{0.75em}
To provide intuition for the role of concept bases in concept intervention, we prove that properly calibrated concept bases, such as Label bases, allow us to predict concepts based on co-occurrence with expert-provided ground truths. 
This shows how inter-concept relationships can impact downstream task performance.
To achieve this, we first show that the similarity between label vectors is directly related to their rate of co-occurrence. 
We then show how concept bases similar to the label basis can be used to predict concept values.
\begin{theorem}
    Suppose an expert intervenes on $r$ concepts while the other $k-r$ concepts are predicted by a concept predictor $g$. 
    Consider label vectors, $\{\mathbf{v}^{(1)}, \cdots, \mathbf{v}^{(k)}\}$ learnt from $n$ data points. 
    Let the matrix $M \in \mathbb{R}^{k \times k}$ represent the co-occurrence matrix, where $M_{i,j} = P(\mathbf{c}_{j} = 1 | \mathbf{c}_{i} = 1)$, and let $\hat{M}_{i,j} = \frac{\mathbf{v}^{(i)} \cdot \mathbf{v}^{(j)}}{|\mathbf{v}^{(i)}|}$. 
    If we define the distance between co-occurrence matrices as $|M-\hat{M}| := \max_{i,j} |M_{i,j}-\hat{M}_{i,j}|$, and let $\theta := \min_{i, j} M_{i,j}$, then for any $\epsilon \in \mathbb{R}$ and $\delta \in \mathbb{R}$, it must be true that $\mathbf{P}[|M-\hat{M}| \geq \epsilon] \leq \delta$ whenever $n > \frac{3}{\epsilon^2 \theta} \mathrm{ln}(1-(1-\delta)^{\frac{1}{k^2}})$. 
    \label{thm:convergence}
\end{theorem}
This shows that similarities between label vectors converge to the co-occurrence of concepts, 
which implies the label bases can be leveraged to predict concepts.
We prove this by bounding $|M-\hat{M}|$ through concentration inequalities. 

Next, we show that small error approximations for concept co-occurrence lead to small error predictions for the presence of concepts, leading to accurate concept interventions.

\begin{theorem}
    Suppose that an expert intervenes on $r$ concepts, while the other $k-r$ concepts are predicted by a concept predictor $g$. 
    Suppose that our prediction for the co-occurrence matrix $M \in \mathbb{R}^{k \times k}$ is corrupted by Gaussian noise, $M' = M+\mathcal{N}(0,\epsilon)$. 
    For any concept $i$, let $\beta_{i} = \mathrm{argmax}_{ 1 \le j \le k} M_{i,j}$ and $\beta'_{i} = \mathrm{argmax}_{1 \le j \le k} M'_{i,j}$.
    Then $\sum_{i=k-r}^{k} M_{i,\beta_{i}} - M_{i,\beta'_{i}} \leq \sum_{i=k-r}^{k} \sum_{j=1}^{r} \Phi(\frac{M_{i,j}-M_{i,\beta{i}}}{\epsilon}) (M_{i,j}-M_{i,\beta_{i}})$, where $\Phi$ is the standard normal CDF.  
    \label{thm:error}
\end{theorem}

Intuitively, this theorem says that when co-occurrence matrices, predicted through label bases, make an error $\epsilon$, concept $i$ goes from having correct prediction probability $M_{i,\beta_{i}}$ to $M_{i,\beta'_{i}}$. 
However, this difference in accuracy is bounded by the structure of the co-occurrence matrix itself, and so $M_{i,\beta_{i}} - M_{i,\beta'_{i}} \leq \sum_{j=1}^{r} \Phi(\frac{M_{i,j}-M_{i,\beta{i}}}{\epsilon}) (M_{i,j}-M_{i,\beta_{i}})$. 
When seen together with Theorem~\ref{thm:convergence}, these theoretical results suggest that well-constructed representations, such as label bases, allow us to have low error (Theorem~\ref{thm:convergence}), and this lets us predict concepts accurately (Theorem~\ref{thm:error}).
Proofs for both theorems can be found in Appendix~\ref{sec:numerics}. 


\begin{figure*}
    \centering
    \includegraphics[width=0.9\textwidth]{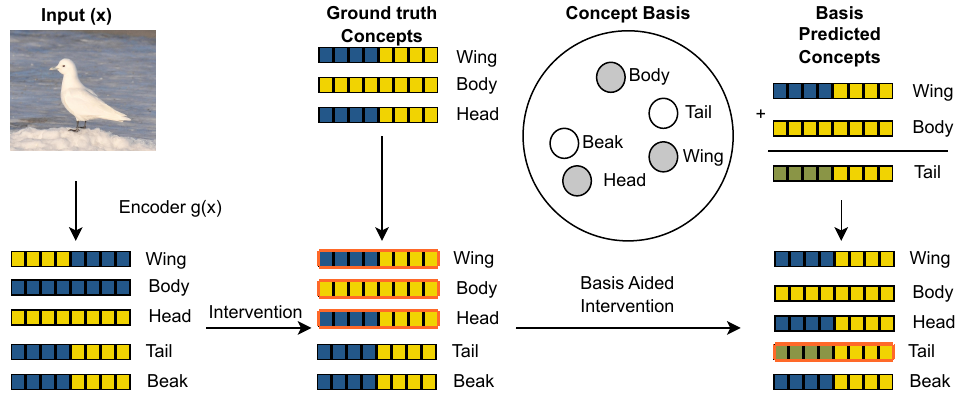}
    \vspace*{-2mm}
    \caption{
    We leverage similarities between concept representations to improve concept interventions. 
    We first predict a set of concept representations for input $\mathbf{x}$ (e.g., vectors in CEM) using concept predictor $g$. 
    Then, if an expert intervenes on $r=3$ of these concepts, ``\textit{wing}'', ``\textit{body}'', ``\textit{head}'', represented by grey circles in the concept basis, we predict ``\textit{tail}'' by finding its $q=2$ nearest neighbours
        in the concept basis (``\textit{wing}'' and ``\textit{body}'') and combining these concepts' representations.
    }
    \label{fig:intervention_diagram}
\end{figure*}

\textbf{Basis Aided Concept Intervention}  \hspace{0.75em}
Motivated by our theoretical results, we develop a novel algorithm for ``basis-aided intervention'', which leverages inter-concept relationships to improve concept intervention accuracy (detailed in Algorithm~\ref{alg:concepts} and Figure~\ref{fig:intervention_diagram}).
We leverage the similarity between concept vectors to impute concept predictions based on expert-provided concepts. 
To predict a concept $j$ for a data point $i$, we leverage the $q$ most similar \textit{intervened} concepts, measuring the similarity of concepts through the distance between concept vectors, $\delta_{v}(\mathbf{v}^{(j)},\mathbf{v}^{(j')})$. 
For each unintervened concept $j$, we predict the concept value $\textbf{c}^{(i)}_j$ by leveraging concept representations similar to concept $j$. 
Formally, let $I(j)$ be the set of the $q$ most similar concepts to concept $j$ that were also intervened upon. 
Our predicted concept is then $\hat{c}_{j}^{(i)} := \frac{1}{|I(j)|} \sum_{j^{\prime} \in I(j)}  \mathbf{c}^{(i)}_{j^{\prime}}$, which is used to make task predictions. 
When $|I(j)| = 0$, we rely on the concept predictor to predict concept values, $\hat{c}^{(i)}_{j} = g(\mathbf{x}^{(i)})_{j}$.
\vspace*{-3mm}
\begin{algorithm}[h]
   \caption{Basis Aided Concept Intervention}
   \label{alg:example}
\begin{algorithmic}
    \STATE {\bfseries Input:} Label predictor $f$, concept basis $B=\{\mathbf{v}^{(1)} \cdots \mathbf{v}^{(k)}\}$, $r$ concept values $\{\mathbf{c}_{1}^{(i)} \cdots \mathbf{c}_{r}^{(i)}\}$
    \STATE {\bfseries Output:} Predicted label $\hat{y}^{(i)}$ 
    \STATE Let $\hat{c}_{j}^{(i)} := \mathbf{c}_{j}^{(i)}$ for $1 \leq j \leq r$
    \STATE Let $s_{j,j'} := \delta_{v}(\mathbf{v}^{(j)}, \mathbf{v}^{(j')})$ for $1 \leq j,j' \leq k$
    \STATE Let $s_{j,j} = \infty, \; \forall \; j \in \{1,\cdots, k\}$
        
    \FOR{$j=r+1$ {\bfseries to} $k$}
        \STATE Let $I(j)$ be the set of indices corresponding to the $q$ smallest values of $\{s_{j,1} \cdots s_{j,r}\}$
        \STATE Let $\hat{c}_{j}^{(i)} := \frac{1}{|I(j)|} \sum_{j^{\prime} \in I(j)}  \mathbf{c}^{(i)}_{j^{\prime}}$
    \ENDFOR
    
    \STATE {\bfseries Return:} $f(\{\hat{c}_{1}^{(i)} \cdots \hat{c}_{k}^{(i)}\})$
\end{algorithmic}
\label{alg:concepts}
\end{algorithm}
\vspace*{-3mm}

\subsection{Empirical Performance}
\paragraph{Experimental Setup}
We evaluate Algorithm~\ref{alg:concepts} by analysing its concept intervention accuracy on the MNIST, CUB, and dSprites datasets (we place our CheXpert evaluation in Appendix~\ref{sec:intervention_ablation} due to the minimal impact of concept interventions). 
For all datasets, we train a CEM model and place details in Appendix~\ref{sec:intervention_ablation}. 


\paragraph{Concept Bases for Concept Interventions}

In Figure~\ref{fig:intervention_dataset} and Figure~\ref{fig:intervention_label} we demonstrate that the quality of representations learnt by concept-based models impacts concept intervention accuracy, as Label bases improve accuracy, while the TCAV and CEM bases hurt accuracy. 
Label bases have the largest impact when $20\%$ to $80\%$ of concepts are known; knowing too few concepts provides too little information for intervention, while knowing most concepts leaves little room for improvement. 
Label bases improve CUB accuracy, outperforming other concept bases, and show the impact of concept bases upon concept intervention.

\begin{figure*}
    \centering
    \subfloat[\centering CUB]{\includegraphics[width=0.3 \textwidth]{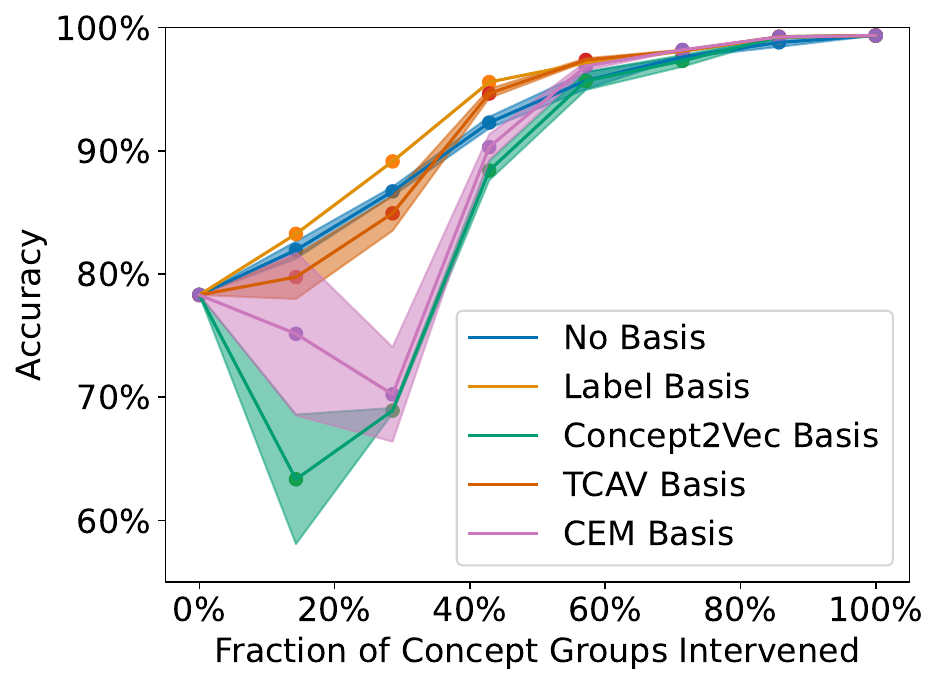}}
    \subfloat[\centering Coloured MNIST]{\includegraphics[width=0.3 \textwidth]{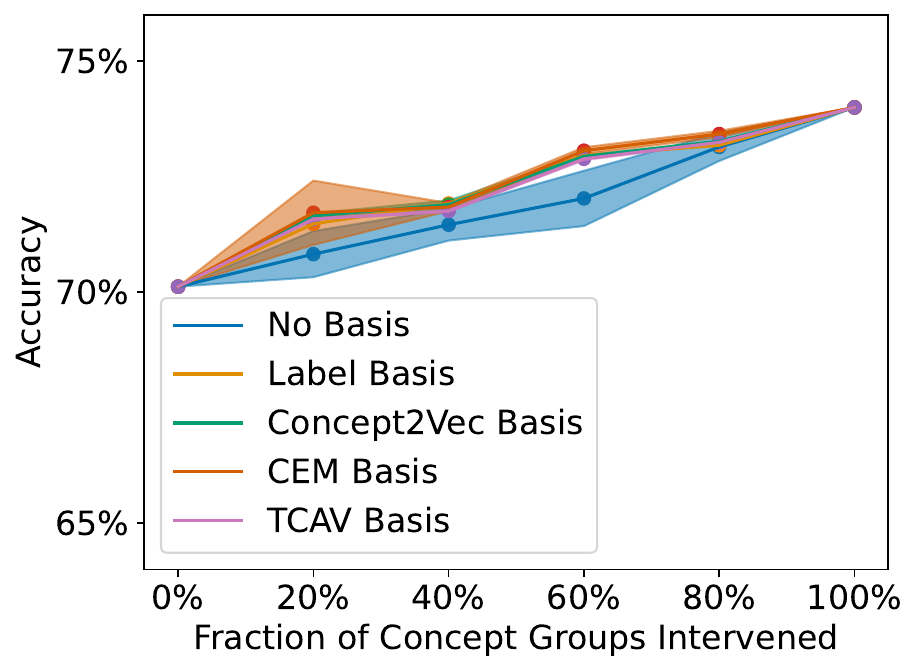}}
    \subfloat[\centering dSprites]{\includegraphics[width=0.3 \textwidth]{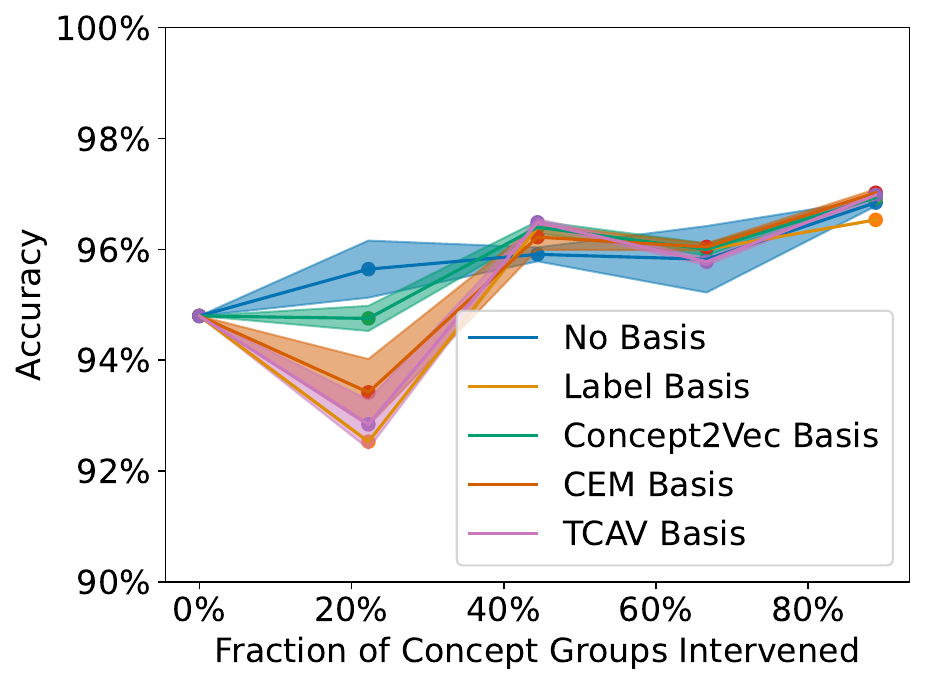}}  
    \caption{
        On the CUB and MNIST datasets, label bases improve concept intervention accuracy when compared with interventions made without concept bases or using other concept bases. Additionally, poorly constructed concept bases (TCAV, CEM) hurt accuracy by up to $10\%$ in CUB. 
    }
        \label{fig:intervention_dataset}
\end{figure*}

\begin{figure}
\centering
\includegraphics[width=0.3\textwidth]{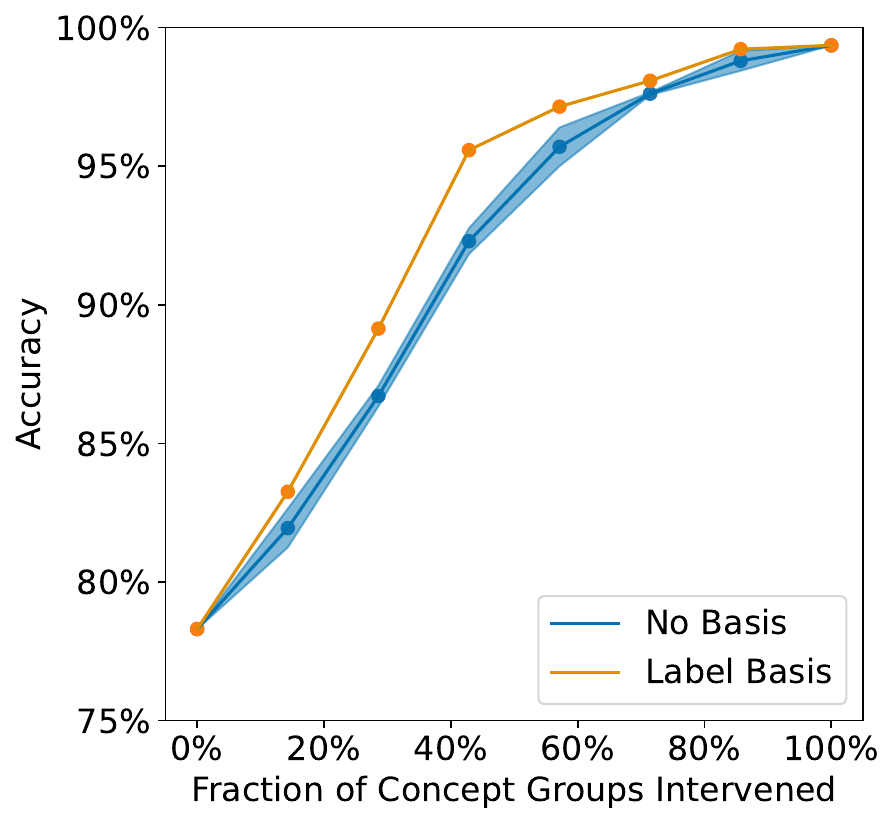}
\caption{Label bases improve concept intervention accuracy on CUB when the fraction of concept groups intervened is from $20\%$ to $80\%$. This demonstrates the potential for inter-concept relationships to assist with downstream tasks.}
\label{fig:intervention_label}
\end{figure}

For dSprites and coloured MNIST, label bases generally improve accuracy, similar to the trends found in CUB (Figure~\ref{fig:intervention_dataset}). 
dSprites presents more complex concept correlations than either the MNIST or CUB datasets, and in this dataset, label bases improve accuracy compared to the baseline when the number of ground truth concepts is more than $20\%$. 
Similarly, coloured MNIST presents a simple scenario where the inter-concept relationships are simple; label bases pick up on these patterns, resulting in an improvement in accuracy. 
Across datasets, better-performing representations lead to improved concept intervention accuracy, reflecting the utility of learning such representations. 

\paragraph{Training time and Concept Bases}
Finally, we investigate whether label bases can improve computational efficiency for concept intervention. 
We train models for $25$, $50$, and $100$ epochs and measure the impact of label bases on concept intervention accuracy. Our results, shown in Figure~\ref{fig:intervention_label}, suggest that label bases lead to larger accuracy improvements for models trained for fewer epochs. 
Notably, training models for fewer epochs and leveraging label bases improves accuracy more than training models for longer, showing that the gains from concept bases cannot be replicated through additional computational resources. 
If models are deployed in conjunction with experts at test-time, concept bases can save computational resources: models trained for fewer epochs can leverage concept bases, and still be competitive with resource-heavy models. 

\begin{figure}
\centering
\includegraphics[width=0.3\textwidth]{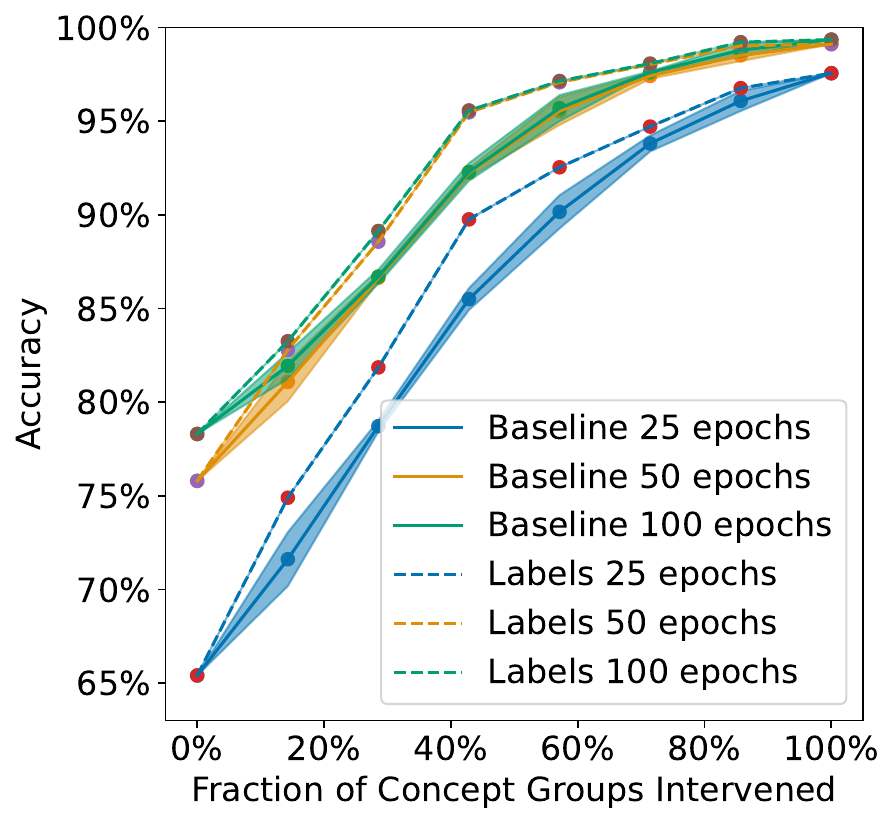}
\caption{In CUB, label bases have a larger impact on concept intervention accuracy for models trained for fewer epochs. The impact of label bases cannot be replaced through additional training epochs, as models trained for 50 epochs with label bases perform better than those trained for 100 epochs without label bases.}
\label{fig:intervention_epochs}
\end{figure}

\section{Discussion and Conclusion}
\paragraph{Limitations}
Throughout our paper, we focus on the application of concept bases across several image-based datasets, focusing on these to capture a diversity of applications. 
However, understanding the performance of such methods across other modalities, such as text and sequence-based data would be useful. 
For text-based data, future work could investigate these models through datasets such as Omniglot~\cite{lake2015human} and CLEVR~\cite{johnson2017clevr}. 
Additionally, a variety of new concept-based models have recently arisen which might have better representations than either CEM or TCAV~\cite{havasi2022addressing, kim2023probabilistic, zarlenga2023learning}; we focus on CEM and TCAV here due to their popularity, but future work could investigate the representations in these new methods.

\paragraph{Conclusion}
\label{sec:conclusion}
In this work, we explored whether representations learnt by popular concept-learning methods capture known inter-concept relationships. 
Unexpectedly, we found that such methods fail to capture these relationships, highlighting an important area for future research. 
Failing to capture these relationships is a missed opportunity for concept-learning methods, as we demonstrated that learning good representations can be useful for downstream applications. 
We theoretically and empirically showed that good representations significantly boost a CEM's receptiveness to concept intervention.
This work highlights the importance of inter-concept relationships and brings forth the need to consider such relationships in future concept-based models.

\newpage
\section*{Impact Statement}
Our paper analyses the foundations of interpretability models within concept-based learning. 
Understanding such models is critical in ensuring safe deployment, so that machine learning models can be trusted and understood, especially in safety-critical situations. 
Our paper works towards this goal by trying to understand how these interpretability methods work and the types of patterns captured by them. 
Such analysis could potentially lead to improved interpretability methods and safer machine learning deployments. 

\section*{Acknowledgements}
The authors would like to thank Katie Collins, George Barbulescu, and Mehtaab Sawhney for their suggestions and discussions on the paper. 
During the time of this work, NR was supported by a Churchill Scholarship, and NR additionally acknowledges support from the NSF GRFP Fellowship. 
MEZ acknowledges support from the Gates Cambridge Trust via a Gates Cambridge Scholarship. 

\bibliography{references}
\bibliographystyle{icml2024}

\appendix 
\newpage 
\section{Metric Details}
\label{sec:metric_details}
\begin{enumerate} 
\item \textbf{Robustness} - We measure the robustness metric by developing a separate, robustness dataset, and comparing inter-concept relationships arising from such a dataset, against a ground truth dataset. 
For example, we compute the robustness of the Label basis by computing the Label basis on the vanilla CUB and robust CUB datasets, compute inter-concept relationships for each, and then compute the similarity of the relationships. 
We develop this robustness dataset by applying two alterations together: we flip concepts at random with probability $p$, and we add Gaussian Noise, with standard deviation $\sigma$. 
We let $p=0.01$ for our experiments and experiment with different values in Appendix~\ref{sec:metric_hyperparameter}, and let $\sigma=50$. 
We do this so that both $\mathbf{c}$ and $\mathbf{x}$ are perturbed. 
\item \textbf{Responsiveness} - We develop the responsiveness metric by randomly altering the image features and concept labels. We let the concept labels be randomly distributed according to a Bernoulli distribution with $p=0.5$, while for the images, we let each pixel be uniformly distributed. We do this to measure whether drastic changes to a dataset's images and concepts impact the underlying inter-concept relationships.   
\item \textbf{Faithfulness} -  For computations of distances between concept bases, we let $t=1$ for the coloured MNIST dataset, while we let $t=3$ for all other datasets. 
We ablate this selection of $t$ in Appendix~\ref{sec:t_choice}. 
For the faithfulness metric, we set $t$ to be 1 for the MNIST dataset, as each digit should be close to its corresponding colour concept, while we set $t=3$ for all other datasets (we explore the impact of $t$ in Appendix~\ref{sec:t_choice}). 
To compute Shapley values, we train a VGG16 concept predictor for all datasets~\citep{simonyan2014very} for 25 epochs with a learning rate of $0.001$ and an Adam optimizer, using this as our concept predictor $g$.
We select these as they avoid biasing towards any particular concept-based model, such as CEM models. 
\end{enumerate} 

\section{Concept Basis Details}
\label{sec:model_details}
\begin{enumerate} 
\item \textbf{Concept2Vec} - We learn representations for Concept2Vec using the Skipgram architecture~\cite{mikolov2013efficient}. we train a model to predict whether two concept pairs come from the same data point or different data points~\citep{mikolov2013efficient}. 
Using this architecture, we develop embeddings for concepts by encouraging co-occurring entities to be nearly parallel in embedding space.
We train this architecture for 25 epochs for each dataset, and we note that additional training epochs did not result in significantly different embeddings.
\item \textbf{CEM} - We train CEM models for $25$ epochs, and let the positive embeddings represent each concept. 
\item \textbf{TCAV} - For the TCAV basis, we use a VGG16 backend~\cite{simonyan2014very} and compute concept activation vectors by comparing them with three reference concepts selected randomly. 
\end{enumerate} 

\section{Dataset Details}
\label{sec:dataset_details}
We provide details on each of our datasets and detail the train-test splits for each. 
For the dSprites and CheXpert datasets, we use 2,500 data points for the training, and 750 for validation and testing. 
For the MNIST dataset, we use 60,000 data points for training and 10,000 data points for validation. 
For CUB, we use 4,796 data points for training, 1,198 for validation, and 5,794 data points for testing. 
We present examples from each dataset in Figure~\ref{fig:dataset}. 

\begin{figure*}[t!]
    \centering
    \begin{subfigure}{90pt}
        \caption*{MNIST}
    \includegraphics[width=90pt]{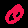}\\[3pt]
    \includegraphics[width=90pt]{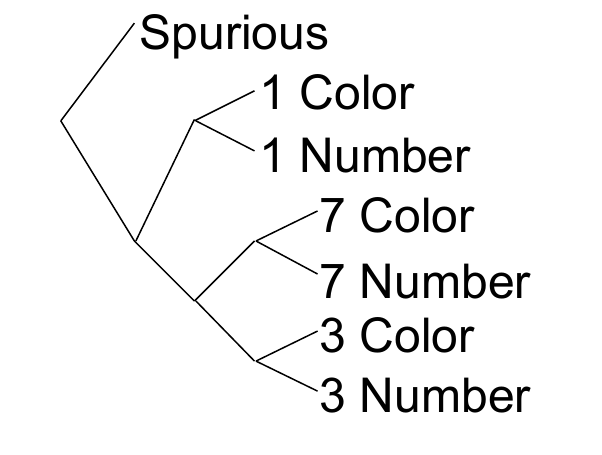}
    \end{subfigure}
    \hfil
    \begin{subfigure}{90pt}
        \caption*{CUB}
    \includegraphics[width=90pt]{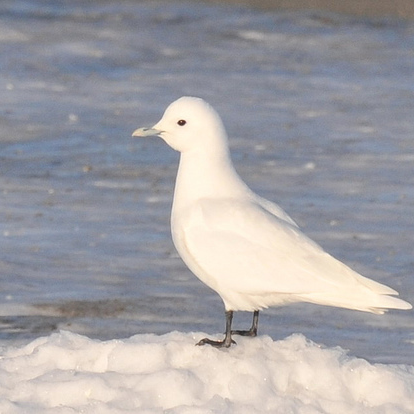}\\[3pt]
    \includegraphics[width=90pt]{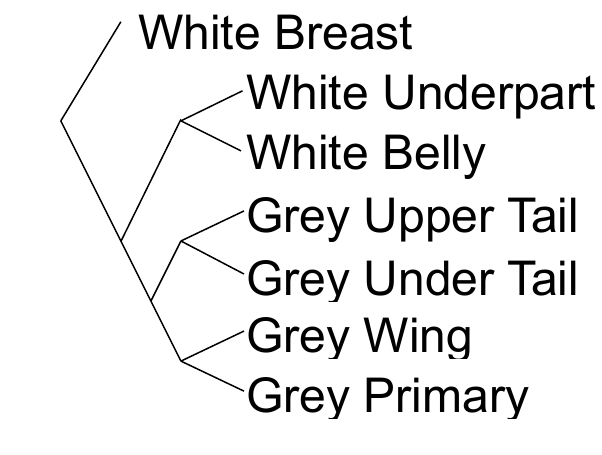}
    \end{subfigure}
    \hfil
    \begin{subfigure}{90pt}
        \caption*{dSprites}
    \includegraphics[width=90pt]{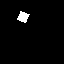}\\[3pt]
    \includegraphics[width=90pt]{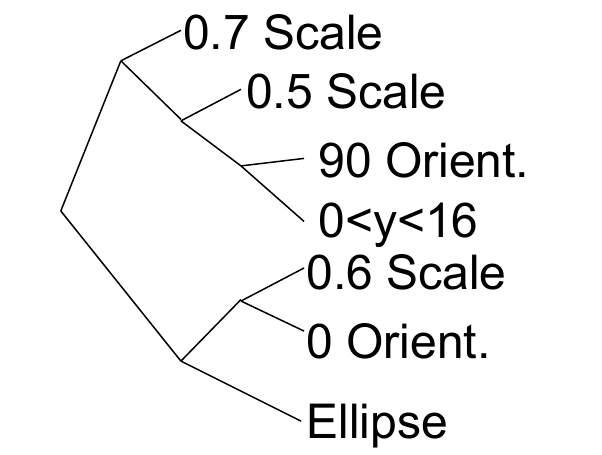}
    \end{subfigure}
    \hfil
    \begin{subfigure}{90pt}
        \caption*{CheXpert}
    \includegraphics[width=90pt]{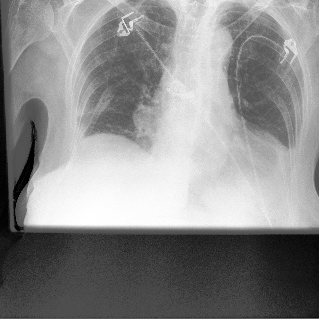}\\[3pt]
    \includegraphics[width=90pt]{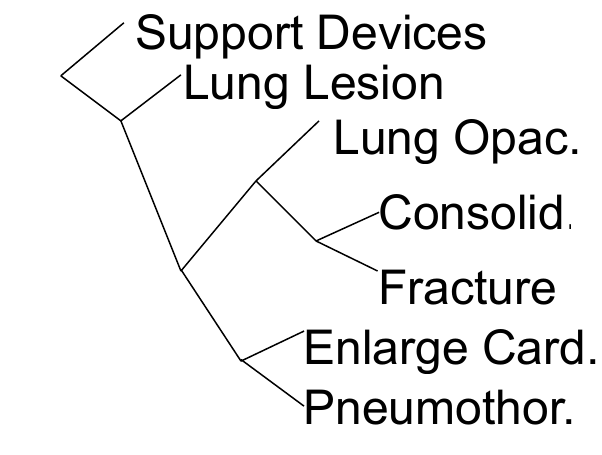}
    \end{subfigure}

    \caption{We visualise examples from the label concept basis across different datasets.}
    \label{fig:dataset}
\end{figure*}

\section{Experimental Details}
For the intervention experiment, we split the 112 concepts available in CUB into 28 concept groups of mutually exclusive concepts following prior work~\citep{koh2020concept}. 
Then, we evaluate CEM's test accuracy across three seeds as we vary the size of the set of intervened concept groups between 0 and 28 in increments of 4, selecting concept groups on each round uniformly at random without replacement. 
For all intervention experiments, we let $r=10$, and we ran all experiments with three different seeds. 
We run our GPU experiments on either an NVIDIA TITAN Xp with 12 GB of GPU RAM on Ubuntu 20.04, or NVIDIA A100-SXM, using at most 8 GB of GPU with Red Hat Linux 8. 
Each run takes at most 2 hours, though most finish in under 45 minutes. 
In total, including preliminary experiments, we run 200-300 hours of GPU experiments.
For concept intervention experiments we use the PyTorch library~\cite{paszke2019pytorch}. 

\section{Impact of Metric Hyperparameters}
\label{sec:metric_hyperparameter}
We evaluate our selection for the rate of concept flipping in the robustness metric. 
We investigate the impact of varying the concept flip rate upon the robustness metric for two bases: the Label basis, and the Concept2Vec basis. 
We select these two methods as they rely only on the concept annotations to compute representations. 
We vary the rate of concept flipping in $\{0.01,0.05,0.1,0.25,0.5\}$ and measure the robustness metric on the CUB dataset. 

\begin{figure}
    \centering
    \includegraphics[width=0.4\textwidth]{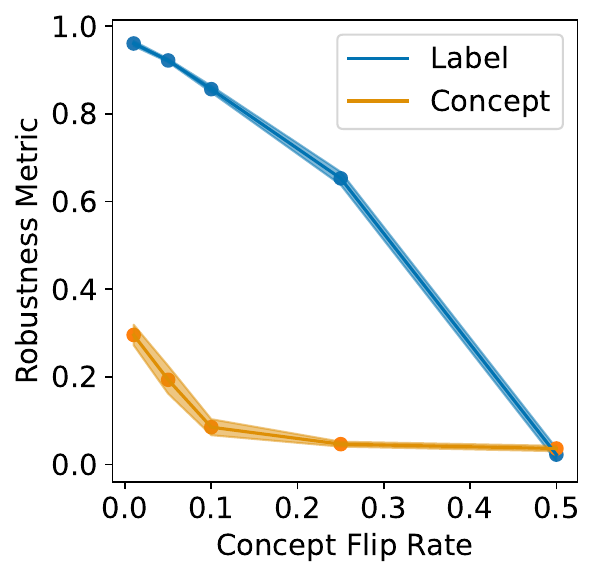}
    \caption{We evaluate the choice of robustness metric hyperparameter by varying the rate of concept flipping, and measuring the resulting robustness metric on the CUB dataset. We find that Label bases maintain inter-concept relationships for values of concept flipping under $0.1$, while Concept2Vec is less robust, as it decreases by $0.2$ in robustness metric. }
    \label{fig:flip_ablation}
\end{figure}

Our results demonstrate that the Label basis is more robust to concept perturbations than the Concept2Vec basis, and this holds across concept flip rates. 
While the Label basis maintains a relatively high robustness metric until flipping 25\% or 50\% of the concepts, flipping only 10\% of the concepts results in the Concept2Vec method having a low robustness score. 
We find that Label bases are more robust than Concept2Vec bases across concept flip rates, showing that our results are not sensitive to the choice of concept flip rate. 

\section{Impact of Distance Metric Choices}
\label{sec:distance_choice}
\begin{figure}
    \centering
    \includegraphics[width=200pt]{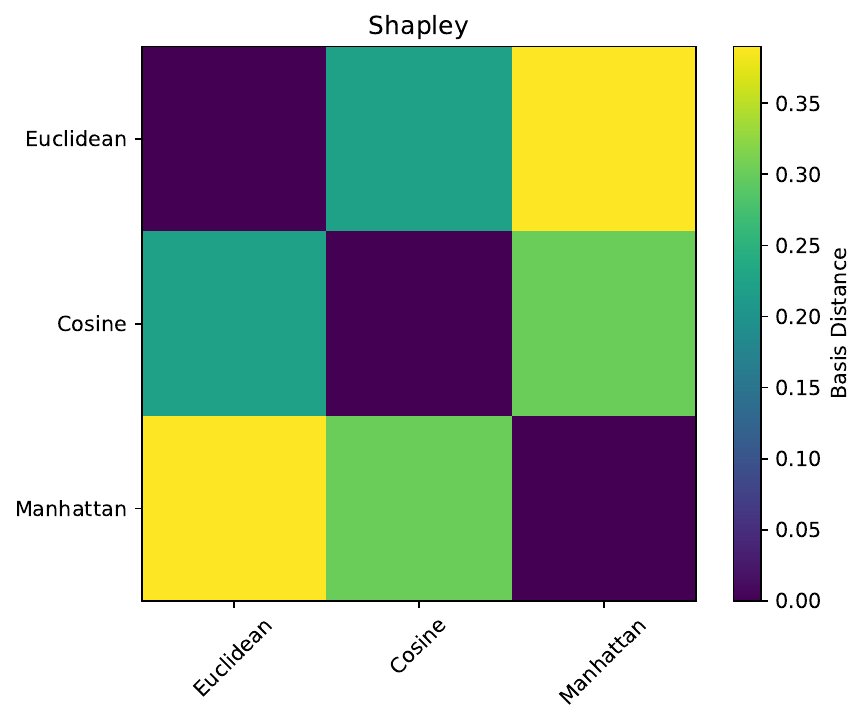}
    \caption{We compare the distance between various concept bases when varying the similarity metric used between concept vectors ($\delta_v$), to understand the sensitivity of our concept bases to $\delta_v$. We see that cosine and Euclidean distances are similar, while Manhattan distances are more dissimilar.}
    \label{fig:distance_metric}
\end{figure}
To understand whether the construction of a concept basis is sensitive to $\delta_v$, the distance metric between concept vectors, we vary the choice of $\delta_v$ and then compare the resulting concept basis. 
We keep the concept vectors the same while varying $\delta_v$, which changes the set of closest concepts. 
We try out three values for this: Euclidean, Manhattan, and cosine distances, which represent three common distance metrics between vectors. 
We compare the distances between bases constructed from each of these metrics in Figure~\ref{fig:distance_metric}, which shows that $\delta_v$ has an impact on basis construction and similarities. 
We find that the Cosine and Euclidean distances are fairly similar, while the Manhattan distance diverges from both of these. 
This indicates that our choice of distance metric, $\delta_{v}$ might impact our results, but only if we chose the Manhattan distance. 

\section{Impact of Concept Distance Choice}
\label{sec:t_choice}
We vary the value of $t$ used during the computation of the faithfulness metric. 
We vary $t$ between 1 and 7 for CUB using the label basis method and plot our results in Figure~\ref{fig:metric_ablation}. We find large jumps in faithfulness from $t=1$ to $t=3$, but then see small increases and decreases afterwards.
\begin{figure}
    \centering
    \includegraphics[width=150pt]{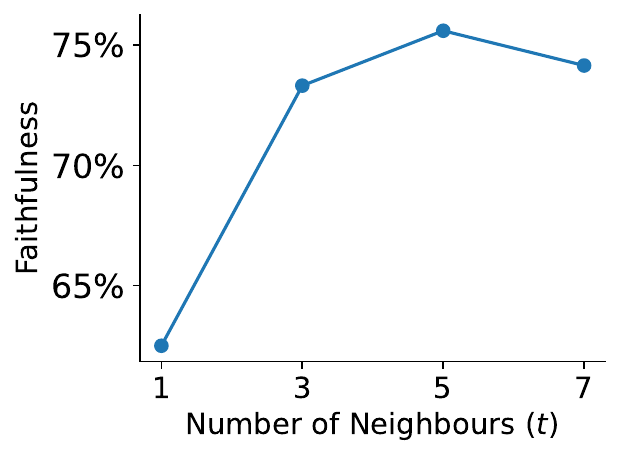}
    \caption{We vary $t$ and compute the faithfulness of the label basis on the CUB dataset. We see that $t$ from 1 to 3 have significant impacts on faithfulness, while values past 3 have little effect.}
    \label{fig:metric_ablation}
\end{figure}

\begin{table*}
    \centering
    \caption{Concept-based models (TCAV, CEM) produce representations that achieve lower scores across all metrics when compared to gold-standard baselines (label).}
    \label{tab:evaluation}
    \begin{adjustbox}{width=\textwidth, center}
        \begin{tabular}{@{}lcccccccc@{}}
            \toprule 
                        & \multicolumn{4}{c}{MNIST}                     & \multicolumn{4}{c}{CUB}               \\ \cmidrule(l){2-5}
                        \cmidrule(l){6-9}
             &
              \multicolumn{1}{l}{Faith. $\uparrow$} &
              \multicolumn{1}{l}{Robust $\uparrow$} &
              \multicolumn{1}{l}{Respons. $\uparrow$} &
              \multicolumn{1}{l}{Stab. $\uparrow$} &
              \multicolumn{1}{l}{Faith. $\uparrow$} &
              \multicolumn{1}{l}{Robust $\uparrow$} &
              \multicolumn{1}{l}{Respons. $\uparrow$} &
              \multicolumn{1}{l}{Stab. $\uparrow$} \\ \cmidrule(l){1-9} 
TCAV & 1.00 $\pm$ 0.00 & 0.12 $\pm$ 0.06 & 0.95 $\pm$ 0.04 & 1.00 $\pm$ 0.00 & 0.08 $\pm$ 0.01 & 0.00 $\pm$ 0.00 & 0.99 $\pm$ 0.00 & 0.01 $\pm$ 0.00 \\
CEM & 0.87 $\pm$ 0.19 & 0.80 $\pm$ 0.08 & 0.95 $\pm$ 0.07 & 0.87 $\pm$ 0.09 & 0.06 $\pm$ 0.00 & 0.02 $\pm$ 0.01 & 0.98 $\pm$ 0.01 & 0.03 $\pm$ 0.00 \\
Concept2Vec & 1.00 $\pm$ 0.00 & 1.00 $\pm$ 0.00 & 0.95 $\pm$ 0.04 & 1.00 $\pm$ 0.00 & 0.41 $\pm$ 0.01 & 0.31 $\pm$ 0.02 & 0.98 $\pm$ 0.01 & 0.29 $\pm$ 0.02 \\
Label & \textbf{1.00 $\pm$ 0.00} & \textbf{1.00 $\pm$ 0.00} & \textbf{1.00 $\pm$ 0.00} & \textbf{1.00 $\pm$ 0.00} & \textbf{0.73 $\pm$ 0.00} & \textbf{0.97 $\pm$ 0.00} & \textbf{0.97 $\pm$ 0.00} & \textbf{1.00 $\pm$ 0.00} \\ \bottomrule
\\
            
                        & \multicolumn{4}{c}{dSprites}                     & \multicolumn{4}{c}{CheXpert}               \\ \cmidrule(l){2-5}
                        \cmidrule(l){6-9}
             &
              \multicolumn{1}{l}{Faith. $\uparrow$} &
              \multicolumn{1}{l}{Robust $\uparrow$} &
              \multicolumn{1}{l}{Respons. $\uparrow$} &
              \multicolumn{1}{l}{Stab. $\uparrow$} &
              \multicolumn{1}{l}{Faith. $\uparrow$} &
              \multicolumn{1}{l}{Robust $\uparrow$} &
              \multicolumn{1}{l}{Respons. $\uparrow$} &
              \multicolumn{1}{l}{Stab. $\uparrow$} \\ \cmidrule(l){1-9} 
TCAV & 0.30 $\pm$ 0.01 & \textbf{0.40 $\pm$ 0.05} & 0.81 $\pm$ 0.05 & 0.56 $\pm$ 0.04 & 0.47 $\pm$ 0.03 & 0.21 $\pm$ 0.05 & 0.76 $\pm$ 0.08 & 0.26 $\pm$ 0.07 \\
CEM & 0.27 $\pm$ 0.01 & 0.24 $\pm$ 0.16 & 0.88 $\pm$ 0.05 & 0.71 $\pm$ 0.03 & 0.41 $\pm$ 0.05 & 0.50 $\pm$ 0.12 & 0.73 $\pm$ 0.13 & 0.44 $\pm$ 0.07 \\
Concept2Vec & 0.31 $\pm$ 0.03 & 0.12 $\pm$ 0.04 & 0.84 $\pm$ 0.03 & 0.83 $\pm$ 0.03 & 0.44 $\pm$ 0.06 & 0.64 $\pm$ 0.08 & 0.76 $\pm$ 0.03 & 0.43 $\pm$ 0.05 \\
Label & \textbf{0.50 $\pm$ 0.00} & 0.19 $\pm$ 0.00 & \textbf{0.78 $\pm$ 0.00} & \textbf{1.00 $\pm$ 0.00} & \textbf{0.52 $\pm$ 0.00} & \textbf{0.92 $\pm$ 0.00} & \textbf{0.79 $\pm$ 0.00} & \textbf{1.00 $\pm$ 0.00} \\ \bottomrule
                        
        \end{tabular}
    \end{adjustbox}
\end{table*}

\section{Intervention Theorems}
\label{sec:numerics}
\begin{theorem*}
    Suppose an expert intervenes on $r$ concepts while the other $k-r$ concepts are predicted by a concept predictor $g$. 
    Consider label vectors, $\{\mathbf{v}^{(1)}, \cdots, \mathbf{v}^{(k)}\}$ learnt from $n$ data points. 
    Let the matrix $M \in \mathbb{R}^{k \times k}$ represent the co-occurrence matrix, where $M_{i,j} = P(\mathbf{c}_{j} = 1 | \mathbf{c}_{i} = 1)$, and let $\hat{M}_{i,j} = \frac{\mathbf{v}^{(i)} \cdot \mathbf{v}^{(j)}}{|\mathbf{v}^{(i)}|}$. 
    If we define the distance between co-occurrence matrices as $|M-\hat{M}| := \max_{i,j} |M_{i,j}-\hat{M}_{i,j}|$, and let $\theta := \min_{i, j} M_{i,j}$, then for any $\epsilon \in \mathbb{R}$ and $\delta \in \mathbb{R}$, it must be true that $\mathbb{P}[|M-\hat{M}| \geq \epsilon] \leq \delta$ whenever $n > \frac{3}{\epsilon^2 \theta} \mathrm{ln}(1-(1-\delta)^{\frac{1}{k^2}})$. 
\end{theorem*}
\begin{proof}
    We first introduce the Chernoff bound, which states that $\mathbb{P}[X>(1+\epsilon)\mu] \leq \exp{(\frac{-\mu \epsilon^{2}}{3})}$, where $X$ is a random variable, and $\mu$ is $E[X]$. 
    In our situation, we apply this to the random variable $\hat{M}_{i,j}$ using $n$ samples. $n \hat{M}_{i,j}$ is a Binomial random variable with $\mu = n M_{i,j}$, and therefore
    \begin{equation} 
        \mathbb{P}\Big[n \hat{M_{i,j}}>(1+\epsilon)n M_{i,j}\Big] \leq \exp{\Big(\frac{- n M_{i,j} \epsilon^{2}}{3}\Big)}
    \end{equation} 

    Now, consider the probability that no concept pair differ by more than $\epsilon$; this has probability $(1-\exp{(-n M_{1,1} \frac{\epsilon^2}{3})})(1-\exp{(-n M_{1,1} \frac{\epsilon^2}{3})}) \cdots $. Using our $\theta$ bound simplifies this to demonstrate that 
    \begin{equation}
        1 - \Big(1 - \exp{\big(-n \theta \frac{\epsilon^2}{3}\big)}\Big)^{k^2} < \delta
    \end{equation}

    Simplifying this yields that this occurs whenever $n>\frac{3}{\epsilon^2 \theta} \mathrm{ln}(1-(1-\delta)^{\frac{1}{k^2}})$
\end{proof}

\begin{theorem*}
    Suppose that an expert intervenes on $r$ concepts, while the other $k-r$ concepts are predicted by a concept predictor $g$. 
    Suppose that our prediction for the co-occurrence matrix $M \in \mathbb{R}^{k \times k}$ is corrupted by Gaussian noise, $M' = M+\mathcal{N}(0,\epsilon)$. 
    For any concept $i$, let $\beta_{i} = \mathrm{argmax}_{ 1 \le j \le k} M_{i,j}$ and $\beta'_{i} = \mathrm{argmax}_{1 \le j \le k} M'_{i,j}$.
    Then $\sum_{i=k-r}^{k} M_{i,\beta_{i}} - M_{i,\beta'_{i}} \leq \sum_{i=k-r}^{k} \sum_{j=1}^{r} \Phi(\frac{M_{i,j}-M_{i,\beta{i}}}{\epsilon}) (M_{i,j}-M_{i,\beta_{i}})$, where $\Phi$ is the standard normal CDF.  
\end{theorem*}
\begin{proof}
    Focus on the regret arising from predicting concept $i$. For this, mistakes arise when the ground truth concept $j$ is used instead of the ground truth concept $1$. That is, if $M^\prime = M + \mathcal{N}(0,\epsilon)$, then whenever $M^\prime_{i,j} > M^\prime_{i,1}$, we incur a incur a regret of $M_{i,1}-M_{i,j}$. We upper bound this as the probability that $M_{i,j} + \mathcal{N}(0,\epsilon) \geq M_{i,1} = \Phi(\frac{M_{i,j}-M_{i,1}}{\epsilon})$. We then repeat this summation across all concepts to get our bound. 
\end{proof}

\section{Basis-Aided Intervention Details}
\label{sec:intervention_ablation}
Throughout our experiments, we select $r=10$, though we found that larger r values have a similar impact. 
We evaluate the impact of $r$, and find that for $q \geq 10$, increasing the $r$ value has minimal impact.  
The reason for this is that some ground truth labels are necessary to improve basis-aided intervention; however, past a certain point, the impact of these values saturates and has minimal impact. 

We modify our concept intervention algorithm (Algorithm~\ref{alg:concepts}) slightly for situations where intervened concepts aren't well-correlated with concepts we aim to predict. 
For example, if we aim to predict concept $j$ using similar intervened concepts $I(j)$, then we compute two predictions: the first is through the concept predictor, $g(\mathbf{x}^{(i)})_{j}$ and the second is through the intervened concepts from Algorithm~\ref{alg:concepts}, $\hat{c}_{j}^{(i)}$. 
We then measure the similarity score, $s_{j,j'} = \delta_{v}(\mathbf{v}^{(j)},\mathbf{v}^{(j'}))$, where we set $\delta_{v}$ to be the cosine similarity between vectors due to its natural $-1-1$ range. 
We leverage this to combine the original prediction for concept $j$, $g(\mathbf{x}^{(i)})_{j}$ and the basis-aided prediction for concept $j$, $\hat{c}_{j}^{(i)}$, weighted by the average similarities, $w_{j} = \frac{1}{|I(j)|} \sum_{j' \in I(j)} s_{j,j'}$, so that the final concept prediction is 
\begin{equation}
    (1-w_{j}) \hat{c}_{j}^{(i)} + g(\mathbf{x}^{(i)})_{j}
\end{equation}
We do this to account for situations where concepts no or few intervened concepts are similar to concept $j$, forcing us to rely on the original concept prediction; however, in situations where concepts are sufficiently similar, we can simply use $\hat{c}_{j}^{(i)}$. 

We select the number of training epochs so that concept interventions still have an impact; this reflects real-world scenarios where models are imperfect and can still be assisted by human experts. 
For CUB, we train models for $100$ epochs, and for all models, we select learning rates through manual inspection. 
For MNIST we select a learning rate of $0.001$ while for all other datasets, we select a learning rate of $0.01$. 
At the same time, for MNIST and dSprites, we increase the difficulty of the task by only training models for $1$ epoch on dSprites and $25$ epochs on MNIST while using only 10\% of the dataset. 
In particular, for MNIST and dSprites, we find that training with the full dataset for $100$ epochs leads to perfect accuracy, rendering concept interventions meaningless. 
For CheXpert, we find the opposite situation, where concept intervention seems not to raise accuracy. 
This might potentially be due to computational limits; the CheXpert dataset is large, so we downsample the dataset to $4000$ training data points but are unable to train models where concept intervention helps. 
We leave further investigation of the CheXpert dataset to future work. 

For all datasets, we use the default parameters and choice of loss functions from~\citet{zarlenga2022concept}. 
We use a ResNet architecture for the concept predictor $g$~\cite{he2016deep} and a 2-layer MLP for the label predictor $f$ 
We only vary the learning rate, which we decide to be $0.001$ for the MNIST and dSprites datasets, while we let this be $0.01$ for the CUB and CheXpert datasets. 
We select these numbers through manual experimentation.

\section{Synthetic Analysis of Metrics}
\label{sec:synthetic}
To evaluate our proposed metrics (Section~\ref{sec:desiderata}), we develop a synthetic scenario and demonstrate the use of our metrics to distinguish between two different concept-based models. 
We develop a synthetic dataset so that we can control the inter-concept relationships.  
We consider a dataset with $\mathbf{x} \in [0,1]^{2}$, and 4 concepts. 
$\mathbf{x}$ is distributed so that, with probability $\frac{1}{3}$, both $\mathbf{x}_{1}$ and $\mathbf{x}_{2}$ are less than $\frac{1}{4}$, and with probability $\frac{1}{3}$, both $\mathbf{x}_{1}$ and $\mathbf{x}_{2}$ are greater than $\frac{3}{4}$. 
The remainder of the time, one of $\{\mathbf{x}_{1},\mathbf{x}_{2}\}$ is less than $\frac{1}{4}$ and the other is more than $\frac{3}{4}$. 
The first two concepts determine whether $ \mathbf{x}_{1} \leq \frac{1}{4}$ and $\mathbf{x}_{1} \geq \frac{3}{4}$, while the last two concepts determine whether $\mathbf{x}_{2} \leq \frac{1}{4}$ and $\mathbf{x}_{2} \geq \frac{3}{4}$. 
Concept-based models in this scenario aim to predict two things: $y_{1} = \mathrm{min}(\mathbf{x}_{1},\mathbf{x}_{2}) \leq \frac{1}{4}$ and $y_{2} = \mathrm{max}(\mathbf{x}_{1},\mathbf{x}_{2}) \geq \frac{3}{4}$. 
We note that the tasks require information from both concepts, and that the concepts are correlated, so that $\mathbf{x}_{1} \geq \frac{3}{4}$ increases the chance that $\mathbf{x}_{2} \geq \frac{3}{4}$. 

We leverage the metrics to distinguish between two concept-based models. 
The first is a linear predictor which leverages the features of $\mathbf{x}$ and a random linear combination of the concepts, $\hat{y} = \sum_{i=1}^{m} v_{i} \mathbf{x}_{i} + \sum_{i=1}^{m} u_{i} c_{i}$ where the $u_{i}$ are uniformly distributed between 0 and 1. 
The second is a linear predictor for each of the two tasks based on the presence of each concept, $\hat{y} = \sum_{i=1}^{k} w_{i} c_{i}$. 
We denote these two methods as ``random'' and ``correct.''
The representations for each model are the collection of weights, $u_{i}$ or $w_{i}$, so that each concept is represented with two weights (one for each task). 
The second concept-based model takes advantage of the available concepts, and therefore, better captures the inter-concept relationships present. 
To confirm this hypothesis, we compare the stability of the inter-concept relationships captured, along with the robustness and responsiveness to random noise.

\begin{figure*}
    \centering
    \subfloat[\centering Stability]{\includegraphics[width=0.35 \textwidth]{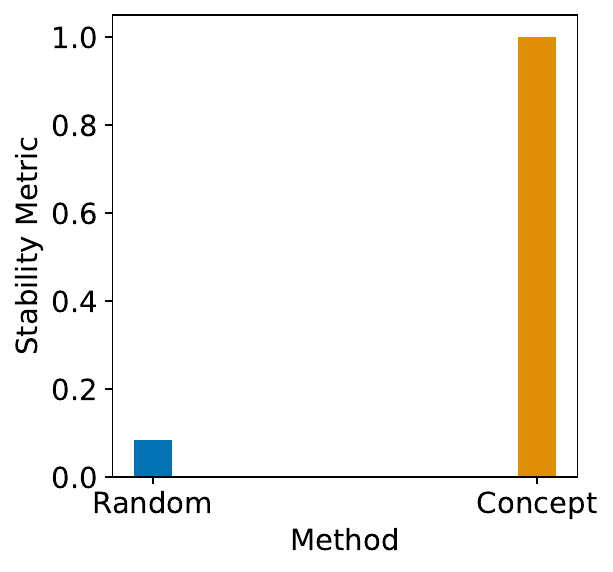}}
    \subfloat[\centering Robustness]{\includegraphics[width=0.35 \textwidth]{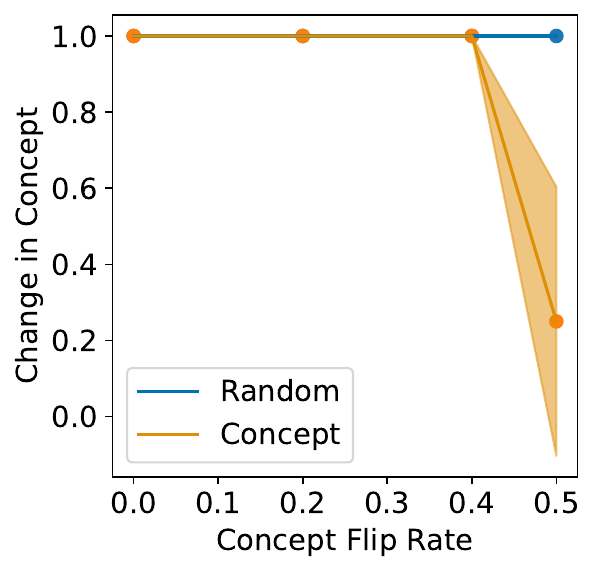}}

    \caption{On our synthetic dataset, the random concept-based model does worse on the stability metric compared to the concept-based model which leverages the available concept information, confirming the efficacy of the stability metric. Additionally, we see that the random concept-based model fails to change predictions even under the presence of heavy noise, showing the necessity for the responsiveness metric.}
    \label{fig:synthetic}
\end{figure*}

We find that the ``random'' concept-based method performs worse according to the stability metric when compared to the concept-based method which correctly leverages concept information (Figure~\ref{fig:synthetic}). 
This demonstrates the use of our stability metric as a way to evaluate concept-based models. 
Additionally, we find that, while both methods maintain the same relationships under a small amount of noise, we find that with increasing amounts of noise, only the ``correct'' concept-based method changes their inter-concept relationships. 
That is, while both methods have a similar robustness metric, the two differ in the responsiveness metric, as the ``random'' concept-based method fails to respond to significant dataset perturbations. 
These experiments demonstrate the ability of the stability, robustness, and responsiveness metrics to distinguish between concept-based models in a controlled dataset.  

\section{Comparison with Concept Leakage Metrics}
\label{sec:ois}
To better understand what our proposed metrics are measuring, we compare each of these metrics to the Oracle Impurity Score (OIS), a metric designed to measure the level of inter-concept leakage~\cite{zarlenga2023towards}. 
OIS is computed by measuring how well concept $i$ can be predicted from the representation for concept $j$, comparing this to a ground-truth oracle based on the true concept values. 
By comparing with the OIS metric, we can better understand whether models which capture inter-concept relationships also have low leakage. 
We measure the OIS scores across all datasets for the CEM and Label methods, selecting these two because we can compute concept representations on a per-data point level. 

\begin{figure}
    \centering
    \includegraphics[width=0.4\textwidth]{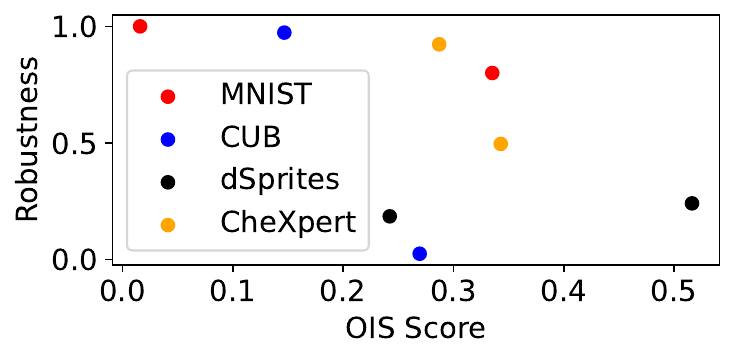}
    \caption{We compare OIS and robustness for the CEM and Label bases and find that across most datasets, more robust representations correspond to better (lower) oracle impurity scores (OIS). This shows that the robustness metric captures the quality of the representation itself.  }
    \label{fig:ois_scatter}
\end{figure}

We find that, across a majority of datasets, methods which exhibit a higher robustness metric additionally have a lower (better) OIS. 
We additionally find that the Label basis has both a better robustness metric and an OIS across a majority of datasets, which further demonstrates the use of Label bases as a baseline. 
The alignment with the OIS implies that our metrics capture fundamental properties of representations, and thereby can be used to better understand concept-based methods.

\section{Impact of Concept Correlations when Captured by Models}
\label{sec:mnist_correlation}
To further understand the impact of inter-concept relationships on downstream accuracy, we investigate how varying the level of concept correlation within a dataset impacts task accuracy. 
We demonstrate that models can leverage inter-concept relationships to improve performance on downstream tasks. 
Our analysis in Section~\ref{sec:experiments} demonstrated that all concept-based models exhibit stable and robust representations on the MNIST dataset. 
Using this information, we analyze whether such an understanding allows for higher task accuracy by leveraging inter-concept relationships. 
We vary the concept correlation between the number and colour concepts, so that the number and colour concepts agree in an $q$-fraction of the examples, randomizing over the colour in other examples. 
Because CEM models can capture inter-concept relationships, we believe that this should allow models to perform better when the strength of inter-concept correlations increases. 
To test this, we vary $q$ in $\{0\%,20\%,40\%,60\%,80\%,100\%\}$, and measure the label and concept accuracy, along with the concept intervention accuracy. 

\begin{figure*}
    \centering
    \subfloat[\centering Concept Acc.]{\includegraphics[width=0.3 \textwidth]{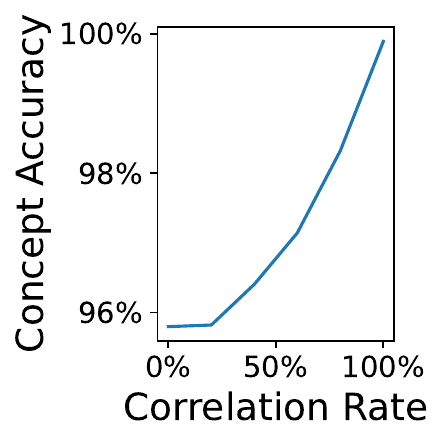}}
    \subfloat[\centering Label Acc.]{\includegraphics[width=0.3 \textwidth]{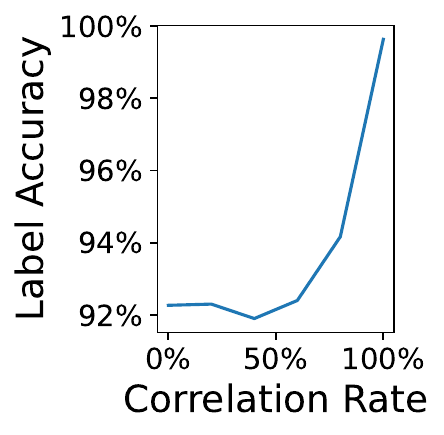}}
    \subfloat[\centering Intervention Acc.]{\includegraphics[width=0.3 \textwidth]{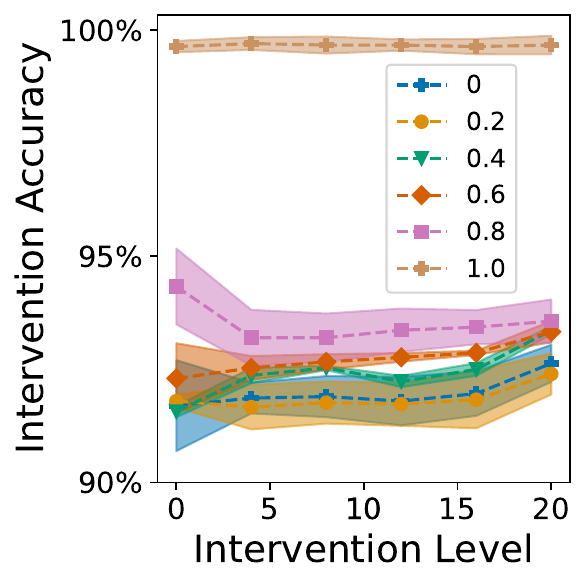}}

    \caption{Increasing correlations between concepts lead to higher accuracy for CEM models trained on the coloured MNIST dataset. We increase the correlation between the number and colour concepts from $0\%$ to $100\%$, with the correlation rate denoting the frequency at which number and colour concepts align. We similarly find that increasing the level of concept correlation increases concept intervention accuracy, across the level of intervention.}
    \label{fig:mnist_correlation}
\end{figure*}

We find that, as expected, increasing the amount of concept correlation between number and colour concepts leads to improved accuracy, for concept, task, and intervention accuracies (Figure~\ref{fig:mnist_correlation}). 
Such an effect is unsurprising, as models which understand the correlation between number and colour concepts can predict the number from the colour concept. 
This provides two sources of information from which MNIST models can predict concept and task labels, which leads to higher accuracy. 
Additionally, understanding inter-concept and concept-task relationships is necessary for improved concept intervention performance; if models are unresponsive to concept imputations, then concept interventions would have no impact on accuracy. 
We find that the efficacy of concept interventions increases with increases in inter-concept correlations, showing that models can pick up on stronger inter-concept correlations. 

\end{document}